\documentclass[conference,onecolumn]{IEEEtran}
\IEEEoverridecommandlockouts

\usepackage{microtype}
\usepackage{graphicx}
\usepackage{subfigure}
\usepackage{booktabs} 

\usepackage{hyperref}

\usepackage{tikz}
\usetikzlibrary{calc,shapes,positioning}
\usetikzlibrary{arrows}

\usepackage{bm}
\usepackage{mystyle}
\usepackage{xr}

\begin{document}
\title{$\alpha$ Belief Propagation for Approximate Inference}
\author{
  \IEEEauthorblockN{
    Dong Liu, Minh Th\`{a}nh Vu, Zuxing Li, and Lars K. Rasmussen}

  \IEEEauthorblockA {KTH Royal Institute of Technology, Stockholm, Sweden. }

  \IEEEauthorblockA{
    e-mail: \{doli, mtvu, zuxing, lkra\}@kth.se}
}

\maketitle

\begin{abstract}
  Belief propagation (BP) algorithm is a widely used message-passing method for inference in graphical models. BP on loop-free graphs converges in linear time. But for graphs with loops, BP's performance is uncertain, and the understanding of its solution is limited. To gain a better understanding of BP in general graphs, we derive an interpretable belief propagation algorithm that is motivated by minimization of a localized $\alpha$-divergence. We term this algorithm as $\alpha$ belief propagation ($\alpha$-BP). It turns out that $\alpha$-BP generalizes standard BP. In addition, this work studies the convergence properties of $\alpha$-BP. We prove and offer the convergence conditions for $\alpha$-BP. Experimental simulations on random graphs validate our theoretical results. The application of $\alpha$-BP to practical problems is also demonstrated.
  
\end{abstract}

\section{Introduction}\label{sec:introduction}
Bayesian inference provides a general mathematical framework for many learning tasks such as classification, denoising, object detection, and signal detection. The wide applications include but not limited to imaging processing \cite{zhang2013denoise}, multi-input-multi-output (MIMO) signal detection in digital communication \cite{cespedes2014ep,jeon2015optimality}, inference on structured lattice \cite{10.2307/25651244}, machine learning  \cite{2018arXiv180607066M, Lin:2015:DLM:2969239.2969280, yoon2019inferenceGraph}.
Generally speaking, a core problem to these applications is the inference on statistical properties of a (hidden) variable $\bm{x} = (x_1,\dots,x_N)$ that usually can not be observed.
Specifically, practical interests usually include computing the most probable state $\bm{x}$ given joint probability $p(\bm{x})$, or marginal probability $p(\bm{x}_c)$, where $\bm{x}_c$ is subset of $\bm{x}$. Naively searching for solutions of these inference problems can be prohibitively expensive in computation.

Probabilistic graphical models defined as structured graphs provide a framework for modeling the complex statistic dependency between random variables. Graphical models are only useful when combined with efficient algorithms. Belief propagation (BP) is a fundamental message-passing algorithm used with graphical models. BP locally exchanges beliefs (statistical information) between nodes \cite{kschischang2001factor_graph, Bishop:2006:PRM:1162264}. BP can solve inference problems in linear-time exactly when graphs are loop-free or tree-structured \cite{kschischang2001factor_graph}. However, many real-world signals are naturally modeled by graph representations with loops. Although BP is still a practical method to do inference approximately (loopy BP) by running it as if there were no loop, its performance varies from case to case and is not guaranteed in general. A direct workaround method to this problem is the general belief propagation \cite{Yedidia:2000:GBP:3008751.3008848} or the junction tree algorithm \cite{wainwright2008graphical}, which try to cluster multiple nodes into super nodes to eliminate loops. The difficulty then lies in the graph modification and inference within super nodes when graphs are dense.

Apart from the practical performance issues of BP in loopy graphs, the understanding of it is also limited. \cite{Yedidia:2000:GBP:3008751.3008848} shows that BP in loopy graphs approaches to a stationary point of the Bethe free energy approximately. Based on this understanding, variants of BP are derived to improve BP. For instance, fractional BP in \cite{Wiegerinck:2002:FBP:2968618.2968673} applies a correction coefficient to each factor; generalized BP in \cite{Yedidia:2000:GBP:3008751.3008848} propagates belief between different regions of a graph; and damping BP in \cite{Pretti2005damping} updates beliefs by combining old and new beliefs. Another track falls to the variational method framework, introduced by Opper and Winther \cite{Opper:2000:GPC:1121900.1121911} and Minka \cite{Minka:2001:EPA:647235.720257, Minka:2001:FAA:935427}, namely expectation propagation (EP). In EP, a simpler factorized distribution defined in exponential distribution family is used to approximate the original complex distribution, and an intuitive factor-wise refinement procedure is used to find such an approximate distribution. The method intuitively minimizes a localized Kullback-Leibler (KL) divergence. This is discussed further in \cite{divergence-measures-and-message-passing} and  shows unifying view of message passing algorithms. The following work, stochastic EP by \cite{yingzhen2015sep}, explores EP's variant method for applications to large dataset.

Due to the fundamental role of BP for probabilistic inference and related applications, research of seeking insight of BP performance and study on its convergence have been constantly carried out.  \cite{weiss2000correctness} presents the convergence condition of BP in graphs containing a single loop. Work in \cite{heskes2004uniqueness} analyzes the Bethe free energy and offers sufficient conditions on uniqueness of BP fixed point.
{Closely related to our work, \cite{mooij2012sufficient-conditions} studies the sufficient conditions for BP convergence to a unique fixed point (as shall be seen in our paper, our convergence analysis is on a message-passing method that generalizes BP).}
\cite{nima2013stochasticBP} proposes a BP algorithm for high-dimensional discrete space and gives the convergence conditions of it. \cite{frederic2019fast} shows that BP can converge to global optima of Bethe energy when BP runs in Ising models that are ferromagnetic (neighboring nodes prefer to be aligned).

There are also works trying to give insight on variant methods of BP. Namely,
\cite{du2017convergenceBP,malioutov2006walk-sums} studies the convergence condition of Gaussian BP inference over distributed linear Gaussian models. \cite{roosta2008reweighed_sum_product} gives the convergence analysis of a reweighted BP algorithm, and offers the necessary and sufficient condition for subclasses of homogeneous graphical models (with identical potentials).

In this work, to gain better understanding of BP in general graphs, we take the path of variational methods to develop an interpretable variant of BP, which we refer to as $\alpha$-BP.
The intuition of $\alpha$-BP starts with a surrogate distribution $q(\bm{x})$. $q(\bm{x})$ is assumed to be fully factorized and each factor of $q(\bm{x})$ represents a message in the graphical model representation of a target distribution $p(\bm{x})$. We derive a message-passing rule that is induced by minimizing a localized $\alpha$-divergence. The merits of $\alpha$-BP are as follows:
\begin{itemize}
\item[a.]{$\alpha$-BP is derived intuitively as localized minimization of $\alpha$-divergence between original distribution $p$ and surrogate distribution $q$.}
\item[b.]{$\alpha$-BP generalizes the standard BP, since the message-passing rule of BP is a special case of $\alpha$-BP.}
\item[c.]{$\alpha$}-BP can outperform BP significantly in full-connected graphs while still maintaining simplicity of BP for inference. 
\end{itemize}
Apart from the algorithm itself, we give the convergence analysis of $\alpha$-BP. Sufficient conditions that guarantee the convergence of $\alpha$-BP to a unique fixed point, are studied and obtained. It turns out that the derived convergence conditions of $\alpha$-BP depend on both the graph and also the value of $\alpha$. This result suggests that proper choice of $\alpha$ can help to guarantee the convergence of $\alpha$-BP. In addition, performance improvement of $\alpha$-BP over standard BP is demonstrated in a practical application.

\section{Preliminary}\label{sec:preliminary}
We provide some preliminaries in this section that are needed in this paper. $\alpha$-divergence is introduced firstly and then a Markov random field is explained.

\subsection{Divergence Measures}
$\alpha$-divergence, introduced in \cite{Zhu95informationgeometric, divergence-measures-and-message-passing}, is a typical way to measure how different two measures characterized by densities $p$ and $q$ are. The definition of $\alpha$-divergence is as follows,
\begin{equation}\label{eq:alpha-divergence}
  \Dd_{\alpha}(p \| q ) = \frac{\int_{\bm{x}} \alpha p(\bm{x}) + (1-\alpha) q (\bm{x}) - p(\bm{x})^{\alpha} q(\bm{x})^{1-\alpha} d\bm{x}}{\alpha(1-\alpha)},
\end{equation}
where $\alpha$ is the parameter, $p$ and $q$ are not necessary to be normalized.

KL divergence as another way of characterizing difference of measures, is closely related with $\alpha$-divergence. KL divergence is defined as
\begin{equation}
  KL(p \| q) = \int p(\bm{x}) \log{\frac{p(\bm{x})}{q(\bm{x})}} d \bm{x}+ \int q(\bm{x}) - p(\bm{x}) d\bm{x},
\end{equation}
where the $\int q(\bm{x}) - p(\bm{x}) d\bm{x}$ is a correction factor to accommodate possibly unnormalized $p$ and $q$. The KL divergence can be seen as a special case of $\alpha$-divergence, by observing $\lim_{\alpha \rightarrow 1}\Dd_{\alpha}(p \| q ) = KL(p\|q)$ and $\lim_{\alpha \rightarrow 0}\Dd_{\alpha}(p \| q ) = KL(q\|p)$ (applying L'H\^opital's rule to \eqref{eq:alpha-divergence}).

Regarding basic properties of divergence measures, both $\alpha$-divergence and KL divergence are zero when $p=q$, and they are non-negative. Denote KL-projection by
\begin{equation}
  \text{proj}[p] = \uargmin{q \in \Ff} KL(p\|q),
\end{equation}
where $\Ff$ is a family of distribution $q$.
According to the stationary point equivalence Theorem in \cite{divergence-measures-and-message-passing}, $\text{proj}[p^{\alpha}q^{1- \alpha}]$ and $\Dd_{\alpha}(p\|q)$ have same stationary points (gradient is zero). A heuristic scheme to find $q^{\ast}$ minimizing $\Dd_{\alpha}(p\|q)$ starts with an initial $q$, and repeatedly updates $q$ via the projection on $\Ff$
\begin{equation}\label{eq:fixed-point-iter}
  q(\bm{x})^{\text{new}}  = \text{proj}[p(\bm{x})^{\alpha}q(\bm{x})^{1-\alpha}].
\end{equation}
This heuristic scheme is a fixed-point iteration. It does not guarantee to converge.

\subsection{Graphic Models}

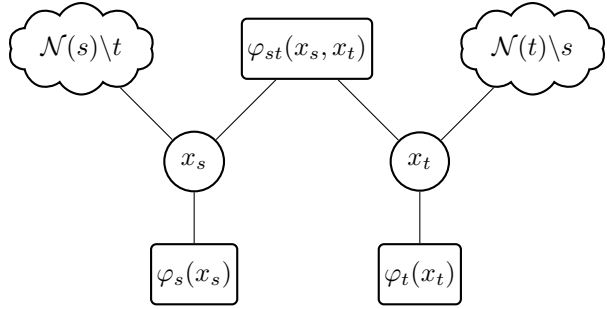
\begin{figure}
  \begin{centering}
    \begin{tikzpicture}
      \tikzstyle{enode} = [thick, draw=black, circle, inner sep = 4pt,  align=center]
      \tikzstyle{nnode} = [thick, rectangle, rounded corners = 2pt,minimum size = 0.8cm,draw,inner sep = 2pt]

      \tikzstyle{cnode} = [thick, cloud, draw,cloud puffs=10, cloud puff arc=120, aspect=2, inner ysep=4pt]

      \node[cnode] (pajk) at (3, 1.5) {$\Nn(t)\backslash s$};
      \node[cnode] (paik) at (-3, 1.5) {$\Nn(s)\backslash t$};

      \node[nnode] (tk) at (0, 1.5) {$\phi_{st}(x_s, x_t)$};
      \node[enode] (xi) at (-1.5 ,0) {$x_s$};
      \node[nnode] (fi) at (-1.5 , -1.5) {$\phi_s(x_s)$};

      \node[enode] (xj) at (1.5 ,0) {$x_t$};
      \node[nnode] (fj) at (1.5 , -1.5) {$\phi_t(x_t)$};

      \draw[-] (xi) to (fi);
      \draw[-] (xi) to (tk);
      \draw[-] (xi) to (paik);

      \draw[-] (xj) to (fj);
      \draw[-] (xj) to (tk);
      \draw[-] (xj) to (pajk);
    \end{tikzpicture}
    \vskip -0.1in
    \caption{Graphic model illustration of $p(\bm{x})$ in \eqref{eq:mrf}.}\label{fig:factor-graph}
  \end{centering}
  \vskip -0.2in
\end{figure}

A graphical model is a probabilistic model that uses a graph to illustrate the statistical dependence between random variables. An undirected graphical model, known as Markov random field (MRF), defines a family of joint probability distributions over random vector $\bm{x} := \left(  x_1, x_2, \cdots,  x_N  \right)$, where each $x_i$ takes values in a discrete finite set $\Aa$. Let us denote the undirected graph of a MRF by $\Gg:=(\Vv, \Ee)$. $\Vv:=\left[ 1 : N \right]$ is the node set associated with the index set of entries of $\bm{x}$. The graph contains undirected edges $\Ee \subset \Vv \times \Vv$, where a pair of $(s, t) \in \Ee$ if and only if nodes $v$ and $u$ are connected by an edge. In addition to the undirected edge set, let us also define the directed edge set of $\Gg$ by $\vec{\Ee}$. We have $\abs{\vec{\Ee}} = 2\abs{\Ee}$, where $\abs{\cdot}$ denotes the carnality.

The joint distribution of $\bm{x}$ can be formulated into a pairwise factorization form as
\begin{equation}\label{eq:mrf}
  p(\bm{x}) \propto \prod_{{s\in \Vv}} \phi_s(x_s) \prod_{(s,t) \in \Ee} \phi_{st}(x_s, x_t),
\end{equation}
where $\phi_{s}: \Aa \rightarrow (0, \infty)$ and $\phi_{st}: \Aa \times \Aa \rightarrow (0, \infty)$ are factor potentials. Relation $\propto$ in \eqref{eq:mrf} indicates that a normalized factor is needed to turn the right-hand side into a distribution.

The factor graph representation of \eqref{eq:mrf} is shown in
Figure~\ref{fig:factor-graph}. In the figure, $\Nn(s)$ is the set of variable nodes neighboring $x_s$ via
pairwise factors, i.e. $\Nn(s) = \left\{ t|(t,s) \in \Ee \right\}$, and $\backslash$ denotes exclusion.

\section{$\alpha$ Belief Propagation Algorithm}\label{sec:alpha-bp-factor-refine}
In this section, we detail the development of the $\alpha$-BP algorithm. We start with defining a surrogate distribution and then use the surrogate distribution to approximate a given distribution. The message passing rule of $\alpha$-BP is derived by solving the distribution approximation problem.

\subsection{Algorithm Development}\label{sec:alpha-bp-a}
We begin with defining a distribution
\begin{equation}
  q(\bm{x}) \propto \prod_{{s\in \Vv}} \tilde{\phi}_s(x_s) \prod_{(s,t) \in \Ee} \tilde{\phi}_{st}(x_s, x_t),
\end{equation}
that is similarly factorized as the joint distribution $p(\bm{x})$. The distribution $q(\bm{x})$ acts as a surrogate distribution of $p(\bm{x})$. The surrogate distribution would be used to estimate inference problems of $p(\bm{x})$. We further choose $q(\bm{x})$ such that it can be fully factorized, which means that $\tilde{\phi}_{s,t}(x_s, x_t)$ can be factorized into product of two independent functions of $x_s, x_t$ respectively. We denote this factorization as
\begin{equation}
  \tilde{\phi}_{s,t}(x_s, x_t) := m_{st}(x_t) m_{ts}(x_s).
\end{equation}
We use the notation $m_{ts}(x_s)$ to denote the factor as a function of $x_s$. $m_{ts}: \Aa \rightarrow (0, \infty)$, serves as the message along directed edge $(t \rightarrow s)$ in our algorithm. Similarly we have factor or message $m_{st}(x_t)$. Then the marginal can be formulated straightforwardly as
\begin{equation}
  q_s(x_s) \propto \tilde{\phi}_s(x_s) \prod_{w\in \Nn(s)} m_{ws}(x_s).
\end{equation}

Now, we are going to use the heuristic scheme as in \eqref{eq:fixed-point-iter} to minimize the information loss by using a fully factorized $q(\bm{x})$ to represent $p(\bm{x})$. The information loss is measured by $\alpha$-divergence $\Dd_{\alpha}(p(\bm{x}) \| q(\bm{x}))$.

We perform a factor-wise refinement procedure to update the factors of $q(\bm{x})$ such that $q(\bm{x})$ approximates $p(\bm{x})$. This approach is similar to the factor-wise refinement procedure of assumed density filtering\cite{ghosh2016assumed,opper1999bayesian} and  expectation propagation \cite{divergence-measures-and-message-passing,Minka:2001:EPA:647235.720257}. Without loss of generality, we begin to refine the factor $\tilde{\phi}_{ts}(x_t, x_s)$ via $\alpha$-divergence characterized by $\alpha$-parameter assigned with $\alpha_{ts}$. Define $q^{\backslash (t,s)}(\bm{x})$ as the product of all other factors excluding $\tilde{\phi}_{ts}(x_t, x_s)$
\begin{align}
  q^{\backslash (t,s)}(\bm{x})
  &= q(\bm{x})/\tilde{\phi}_{ts}(x_t,
    x_s) \propto \prod_{{s\in \Vv}} \tilde{\phi}_s(x_s) \prod_{(v,u) \in
    \Ee\backslash (t,s)}
    \tilde{\phi}_{vu}(x_v, x_u).
\end{align}
We also exclude the factor $\phi_{ts}(x_t, x_s)$ in $p(\bm{x})$ to obtain $p^{\backslash (t,s)}(\bm{x})$. 
Instead of updating $\tilde{\phi}_{ts}(x_t, x_s)$ directly by solving
\begin{equation}
  \!\!\!\uargmin{\tilde{\phi}_{ts}^{\mathrm{new}}(x_t, x_s)}\!\!\!\! \Dd_{\alpha_{ts}}\!\!\left(  p^{\backslash (t,s)}\!(\bm{x})\phi_{ts}(x_t, x_s)\|q^{\backslash (t,s)}\!(\bm{x})\tilde{\phi}_{ts}^{\mathrm{new}}(x_t, x_s)\right),
\end{equation}
we consider the following tractable problem
\begin{equation}\label{eq:alpha-minimize-factor}
  \!\!\!\uargmin{\tilde{\phi}_{ts}^{\mathrm{new}}(x_t, x_s)}\!\!\!\!
  \Dd_{\alpha_{ts}}\!\!\left( q^{\backslash (t,s)}\!(\bm{x}){\phi}_{ts}(x_t, x_s)
    \|q^{\backslash (t,s)}\!(\bm{x})\tilde{\phi}_{ts}^{\mathrm{new}}(x_t, x_s) \right),
\end{equation}
which searches for new factor $\tilde{\phi}_{ts}^{\mathrm{new}}(x_t, x_s)$ such $q$ can approximate $p$ better. In \eqref{eq:alpha-minimize-factor}, $\Dd_{\alpha_{ts}}(\cdot)$ denotes the $\alpha$-divergence with the correcponding parameter $\alpha_{ts}$. Note that the approximation \eqref{eq:alpha-minimize-factor} is accurate when $q^{\backslash (t,s)}(\bm{x})$ is equal to $p^{\backslash (t,s)}(\bm{x})$. 
Using fixed-point update in \eqref{eq:fixed-point-iter}, the problem in \eqref{eq:alpha-minimize-factor} is equivalent to
\begin{align}\label{eq:update-rule}
  \!\!&q^{\backslash (t,s)}(\bm{x})\tilde{\phi}_{ts}^{\mathrm{new}}(x_t, x_s) \propto 
    \mathrm{proj}\left[ q^{\backslash
        (t,s)}(\bm{x}){\phi}_{ts}(x_t, x_s)^{\alpha_{ts}} \tilde{\phi}_{ts}(x_t, x_s)^{1-\alpha_{ts}} \right]. 
\end{align}
\begin{figure}[!t]
  \begin{centering}
    \begin{tikzpicture}
      \tikzstyle{enode} = [thick, draw=black, circle, inner sep = 4pt,  align=center]
      \tikzstyle{nnode} = [thick, rectangle, rounded corners = 2pt,minimum size = 0.8cm,draw,inner sep = 2pt]
      \tikzstyle{cnode} = [thick, cloud, draw,cloud puffs=10, cloud puff arc=120, aspect=2, inner ysep=4pt]

      \node[cnode] (paik) at (-2, 0) {$\Nn(s)$};
      \node[enode] (xi) at (0 ,0) {$x_s$};
      \node[nnode] (fi) at (0 , -1.5) {$\phi_s(x_s)$};
      \node[nnode] (pi) at (2, 0) {$\hat{p}_s(x_s)$};
      \draw[-] (xi) to (fi);
      \draw[-] (xi) to (pi);
      \draw[-] (xi) to (paik);
    \end{tikzpicture}
    \vskip -0.1in
    \caption{Modified graphical model with prior factor.}\label{fig:factor-graph-with-prior}
  \end{centering}
  \vskip -0.2in
\end{figure}
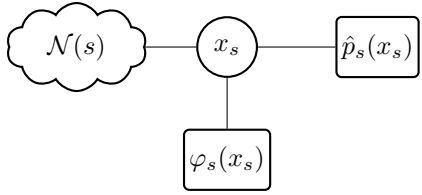

Without loss of generality, we update $m_{ts}$ and define
\begin{equation}\label{eq:new_ts}
  \tilde{\phi}_{ts}^{\mathrm{new}}(x_t, x_s) = m_{ts}^{\mathrm{new}}(x_s) m_{st}(x_t).
\end{equation}
Since KL-projection onto a fully factorized distribution reduces to matching the marginals, substituting \eqref{eq:new_ts} into \eqref{eq:update-rule}, we obtain
\begin{align}\label{eq:message-update}
  &\sum_{\bm{x}\backslash x_s} q^{\backslash (t,s)}(\bm{x}) \tilde{\phi}_{ts}^{\mathrm{new}}(x_t, x_s) \propto \sum_{\bm{x}\backslash x_s} q^{\backslash (t,s)}(\bm{x}) \phi_{ts}(x_t, x_s)^{\alpha_{ts}} \tilde{\phi}_{ts}(x_t, x_s)^{1-\alpha_{ts}}.
\end{align}
We use summation here. But it should be replaced by integral if $\Aa$ is a continuous set.
Solving \eqref{eq:message-update} gives the message passing rule as
\begin{align}\label{eq:message-rule-pairwise}
  {m}^{\text{new}}_{ts}(x_s) \propto  &m_{ts}(x_s)^{1-\alpha_{ts}} \bigg[
                                        \sum_{x_t} \phi_{ts}(x_t, x_s)^{\alpha_{ts}} {m}_{st}(x_t)^{1-\alpha_{ts}}
                                      \tilde{\phi}_t(x_t) \prod_{w\in \Nn(t)\backslash s}m_{wt}(x_t) \bigg].
\end{align}

As for the singleton factor $\tilde{\phi}_t(x_t)$, we can do the refinement procedure on $\tilde{\phi}_t(x_t)$ in the same way as we have done on $\tilde{\phi}_{ts}(x_t, x_s)$. This gives us the update rule of $\tilde{\phi}_t(x_t)$ as
\begin{equation}\label{eq:fix-factor-update}
  \tilde{\phi}_t^{\mathrm{new}}(x_t) \propto \phi_t(x_t)^{\alpha_{t}} \tilde{\phi}_t(x_t)^{1-\alpha_{t}},
\end{equation}
which is the belief from factor $\phi_t(x_t)$ to variable $x_t$. Here $\alpha_t$ is the local assignment of parameter $\alpha$ in $\alpha$-divergence in refining factor $\tilde{\phi}_t(x_t)$. Note, if
we initialize $\tilde{\phi}_t(x_t) = \phi_t(x_t)$, then it remains the
same in all iterations, which makes
\begin{align}\label{eq:message-rule}
  {m}^{\mathrm{new}}_{ts}(x_s) \propto
  & m_{ts}(x_s)^{1-\alpha_{ts}} \bigg[\sum_{x_t} \phi_{ts}(x_t, x_s)^{\alpha_{ts}} {m}_{st}(x_t)^{1-\alpha_{ts}} 
  {\phi}_t(x_t) \prod_{w\in \Nn(t)\backslash s}m_{wt}(x_t) \bigg].
\end{align}
In our notations, a factor potential is undirected, i.e. $\phi_{ts}(x_t, x_s)=\phi_{st}(x_s, x_t)$ for all $(t,s) \in \Ee$. When refining factors with $\alpha$-BP, each factor potential (corresponding to an edge of $\Gg$) can be associated with a difference setting of $\alpha$ value. In addition we also have $\alpha_{ts} = \alpha_{st}$.

\vskip -0.5in
\subsection{Remarks on $\alpha$ Belief Propagation}\label{subsec:remark}

As discussed in Section~\ref{sec:preliminary}, $KL(p\|q)$ is the special case of $\Dd_{\alpha}(p\|q)$ when $\alpha \rightarrow 1$. When restricting $\alpha_{st}=1$ for all $(s,t) \in \Ee$, the message-passing rule in \eqref{eq:message-rule} becomes
\begin{equation}\label{eq:bp-update-rule}
  {m}^{\text{new}}_{ts}(x_s) \propto 
  \sum_{x_t} \phi_{st}(x_s, x_t) {\phi}_t(x_t) \prod_{w\in \Nn(t)\backslash s}m_{wt}(x_t),
\end{equation}
which is exactly the messages of standard BP \cite{Bishop:2006:PRM:1162264}. From this point of view, we can say $\alpha$-BP is a generalization of BP.

From the practical perspective of view, $\alpha$-BP as a meta algorithm can be used with other methods in hybrid way. Inspired by \cite{pseudo_priorBP2010} and assembling methods \cite{James:2014:ISL:2517747}, we can modify the graphical model shown in Figure~\ref{fig:factor-graph} by adding an extra factor potential $\hat{p}_s(x_s)$ to each $x_s$. The extra factor potential $\hat{p}_s(x_s)$ acts as prior information that can be obtained from other methods. In other words, this factor potential stands for our belief from exterior estimation. Then we can run our $\alpha$-BP on the modified graph. The modified graph is shown in Figure~\ref{fig:factor-graph-with-prior}.



Note although mean field method also uses fully-factorized approximation, it is obtained differently from $\alpha$-BP and its factorization differs from that of $\alpha$-BP. In addition, 
$\alpha$-BP is different from standard BP with damping technique. The later case uses message update rule that differs from \eqref{eq:bp-update-rule} slightly by the way of assigning  updated message. Also, $\alpha$-BP differs from the tree-reweighted belief propagation \cite{wainwright2008graphical} by the way of message update rule and also how algorithm is derived. Please refer to the section~\ref{apdx:sec:related-msg-passings} in the supplementary for detailed discussion.

\section{Convergence of $\alpha$-BP with a Binary State Space}\label{sec:cnvg-thm}
In this section we discuss the convergence of $\alpha$-BP. We consider the case of binary $\Aa$, i.e. $\Aa=\left\{ -1, 1 \right\}$. The factor potentials are further detailed as
\begin{align}
  \phi_{st}(x_s, x_t) &= \exp\left\{ \theta_{st}(x_s, x_t)\right\}, \nonumber \\
  \phi_{s}(x_s) &= \exp\left\{ \theta_{s}(x_s) \right\}.
\end{align}
Further assume the symmetric property of potentials
\begin{align}
  \theta_{ts}(x_t, x_s) &= -\theta_{ts}(x_t, -x_s) = -\theta_{ts}(-x_t, x_s), \nonumber\\
  \theta_{s}(x_s) &= - \theta_s(-x_s).
\end{align}

For notation simplicity, we use $\theta_{ts}=\theta_{ts}(1, 1)$ and $\theta_s = \theta_s(1)$. Denote by $\bm{\alpha}$ the vector of all local assignments of parameter $\alpha$, i.e. $\bm{\alpha} = \left(  \alpha_{ts} \right)_{(t ,s) \in \Ee }$, by $\bm{\theta}$ the vector of all parameters of potentials, i.e. $\bm{\theta} = \left(  \theta_{ts} \right)_{(t ,s) \in \Ee }$.  Define a matrix $\bm{M}(\bm{\alpha}, \bm{\theta})$ of size $\abs{\vec{\Ee}} \times \abs{\vec{\Ee}}$, in which its entries are indexed by directed edges $(t\rightarrow s)$, as
\begin{align}\label{eq:matrix_m}
  M_{(t\rightarrow s), (u \rightarrow v)} 
  =
    \begin{cases}
      \abs{ 1 - \alpha_{ts}}, &\quad u=t, v=s, \\
      \abs{ 1 - \alpha_{ts}} \tanh \abs{\alpha_{ts} \theta_{ts}}, & \quad u = s, v =t, \\
      \tanh\abs{\alpha_{ts} \theta_{ts}}, & \quad u \in \Nn(t)\backslash s, v = t, \\
      0, & \quad \mathrm{otherwise}.
    \end{cases}
\end{align}
\begin{thm}\label{thm:normd}
  For an arbitrary pairwise Markov random field over binary variables,
  if the largest singular value of matrix $\bm{M}(\bm{\alpha}, \bm{\theta})$ is less than one,
  $\alpha$-BP converges to a fixed point. The associated fixed point is unique.
\end{thm}

\begin{proof}
  Let us define $z_{ts}$ as the log ratio of belief from node $t$ to node $s$ on two states of $\Aa$, i.e.
  \begin{equation}\label{eq:log-ratio-uv}
    z_{ts} = \log\frac{m_{ts}(1)}{m_{ts}(-1)}.
  \end{equation}

  By combining the local message passing rule in \eqref{eq:message-rule} with \eqref{eq:log-ratio-uv}, we obtain a local update function $F_{ts}: \; \RR^{\abs{\vec{\Ee}}} \rightarrow \RR$ that maps $\bm{z} = \left( z_{ts} \right)_{(t \rightarrow s) \in \Ee}$ to updated $z_{ts}$, which can be expressed as
  \begin{equation}\label{eq:ratio-update}
    F_{ts}(\bm{z}) = (1-\alpha_{ts}) z_{ts} + f_{ts}(\bm{z}),
  \end{equation}
  where
  \begin{equation}
    f_{ts}(\bm{z}) = \log\frac{\exp\left\{ 2 \alpha_{ts} \theta_{ts} + \Delta_{ts}(\bm{z}) \right\}+1}
    {\exp\left\{ \Delta_{ts}(\bm{z}) \right\} +
      \exp\left\{ 2 \alpha_{ts} \theta_{ts} \right\}},
  \end{equation}
  with
  \begin{equation}
    \Delta_{ts}(\bm{z}) = 2 \theta_s + (1 - \alpha_{ts}) z_{st} + \sum_{w\in \Nn(u)\backslash t} z_{wt}.
  \end{equation}

  In the following, we use superscript ${(n)}$ to denotes the $n$-th iteration. Since $f_{ts}$ is continuous on $\RR^{\abs{\vec{\Ee}}}$ and differentiable, we have
  \begin{align}\label{eq:ratio_diff_at_n}
    & z_{ts}^{(n+1)} - z_{ts}^{(n)} \nonumber \\
    =& (1-\alpha_{ts}) (z_{ts}^{(n)} - z_{ts}^{(n-1)}) + f_{ts}(\bm{z}^{(n)}) - f_{ts}(\bm{z}^{(n-1)}) \nonumber \\
    \overset{(a)}{=}& (1-\alpha_{ts}) (z_{ts}^{(n)} - z_{ts}^{(n-1)}) + \nabla f_{ts}(\bm{z}^{\lambda})^{T} (\bm{z}^{(n)} - \bm{z}^{(n-1)}),
  \end{align}
  where $(a)$ follows by the mean-value theorem, $\bm{z}^{\lambda} = \lambda \bm{z}^{(n)} + (1 - \lambda) \bm{z}^{(n-1)}$ for some $\lambda \in (0,1)$, and $\nabla f_{ts}(\bm{z}^{\lambda})$ denotes the gradient of $f_{ts}$ evaluated at $\bm{z}^{\lambda}$. In further details, $\nabla f_{ts}$ is given by
  \begin{equation}\label{eq:F_pd}
    \pd{f_{ts}}{z}=
    \begin{cases}
      ( 1 - \alpha_{ts}) \pd{f_{ts}}{\Delta_{ts}}, &\quad z = z_{st}, \\
      \pd{f_{ts}}{\Delta_{ts}}, &\quad z = z_{wt}, w\in N(t)\backslash s. \\
      0, &\quad \mathrm{otherwise}.  \\
    \end{cases}
  \end{equation}
  
  Our target here is to find the condition to make sequence $\left( z_{ts}^{(n+1)} - z_{ts}^{(n)} \right)$ to converge. To this aim we need to bound the term $\nabla f_{ts}(\bm{z}^{\lambda})^{T} (\bm{z}^{(n)} - \bm{z}^{(n-1)})$ in \eqref{eq:ratio_diff_at_n}. For this purpose, we need two auxiliary functions $H, G: \RR^2 \rightarrow \RR$ from lemma $4$ in \cite{roosta2008reweighed_sum_product}, which are cited herein for completeness
  \begin{align}
    H(\mu; \kappa) &:= \log \frac{exp(\mu + \kappa) +1}{exp(\mu) + exp(\kappa)}, \nonumber \\
    G(\mu; \kappa) &:= \frac{exp(\mu + \kappa)}{exp(\mu + \kappa) + 1} - \frac{exp(\mu)}{exp(\mu) + exp(\kappa)} \nonumber \\
                   &= \frac{\sinh{\kappa}}{\cosh{\kappa} + \cosh{\mu}},
  \end{align}
  where it holds that $\pd{H(\mu; \kappa)}{\mu} = G(\mu; \kappa)$. Further, it holds that
  $\abs{G(\mu; \kappa)} \leq \abs{G(0, \kappa)} = \tanh(\abs{\kappa}/2)$.
  Then we have
  \begin{align}\label{eq:eqv_aux_functions}
    f_{ts}(z) &= H(\Delta_{ts}(\bm{z}) ;2 \alpha_{ts} \theta_{ts}), \nonumber \\
    \pd{f_{ts}}{\Delta_{ts}} &= G(\Delta_{ts}(\bm{z}) ;2 \alpha_{ts} \theta_{ts}),
  \end{align}
  which implies
  \begin{equation}\label{eq:f-bound}
    \bigg|{\pd{f_{ts}}{\Delta_{ts}}}\bigg| \leq \tanh(\alpha_{ts}\theta_{ts}).
  \end{equation}
  Combining \eqref{eq:ratio_diff_at_n}, \eqref{eq:F_pd}, \eqref{eq:eqv_aux_functions} and \eqref{eq:f-bound}, we have
  \begin{align}\label{eq:ratio_cvg_step}
    &\abs{z_{ts}^{(n+1)} - z_{ts}^{(n)}} \nonumber\\
    & = \abs{(1-\alpha_{ts}) (z_{ts}^{(n)} - z_{ts}^{(n-1)}) + \nabla f_{ts}(\bm{z}^{\lambda})^{T} (\bm{z}^{(n)} - \bm{z}^{(n-1)})} \nonumber \\
    & \leq \abs{(1-\alpha_{ts}) (z_{ts}^{(n)} - z_{ts}^{(n-1)})} + \abs{\nabla f_{ts}(\bm{z}^{\lambda})^{T} (\bm{z}^{(n)} - \bm{z}^{(n-1)})} \nonumber \\
    & = \abs{ 1 - \alpha_{ts}} \abs{z_{ts}^{(n)} - z_{ts}^{(n-1)}} + \abs{\nabla f_{ts}(\bm{z}^{\lambda})}^{T} \abs{\bm{z}^{(n)} - \bm{z}^{(n-1)}} \nonumber \\
    & \overset{(a)}{\leq} \abs{ 1 - \alpha_{ts}} \abs{z_{ts}^{(n)} - z_{ts}^{(n-1)}} \nonumber \\
    &+ \abs{ 1 - \alpha_{ts}} \tanh(\abs{\alpha_{ts} \theta_{ts}}) \abs{z_{st}^{(n)} - z_{st}^{(n-1)}} \nonumber \\
    &+ \sum_{w\in N(t)\backslash s} \tanh(\abs{\alpha_{ts} \theta_{ts}}) \abs{z_{wt}^{(n)} - z_{wt}^{(n-1)}},
  \end{align}
  where step $(a)$ holds by applying \eqref{eq:F_pd} and \eqref{eq:f-bound}.

  Concatenating all $(t\rightarrow s) \in \vec{\Ee}$ for inequality \eqref{eq:ratio_cvg_step} gives
  \begin{equation}\label{eq:ratio-update-ineq}
    \abs{\bm{z}^{(n+1)} - \bm{z}^{(n)}} \leq \bm{M}(\alpha, \bm{\theta})\abs{\bm{z}^{(n)} - \bm{z}^{(n-1)}},
  \end{equation}
  where $\bm{M}(\bm{\alpha}, \bm{\theta})$ is defined in \eqref{eq:matrix_m}, and $\leq$ in \eqref{eq:ratio-update-ineq} denotes the element-wise inequality. From \eqref{eq:ratio-update-ineq}, we could further have
  \begin{align}\label{eq:ine-normp-iteration}
    \norm{\bm{z}^{(n+1)} - \bm{z}^{n}}_{p} &\leq \norm{\bm{M}(\alpha, \bm{\theta})\abs{\bm{z}^{(n)} - \bm{z}^{(n-1)}}}_{p},               
  \end{align}
  where $1\leq p < \infty$, and $\norm{\cdot}_p$ denotes the ${\ell}^p$-norm.

  When applying $p=2$ to \eqref{eq:ine-normp-iteration}, we have
  \begin{align}\label{eq:ieq-largest-singular}
    \normd{\bm{z}^{(n+1)} - \bm{z}^{(n)}} &\leq \normd{\bm{M}(\bm{\alpha}, \bm{\theta})\abs{\bm{z}^{(n)} - \bm{z}^{(n-1)}}} \nonumber \\
                                          &\leq \lambda^{\ast}(\bm{M})\normd{\bm{z}^{(n)} - \bm{z}^{(n-1)}},
  \end{align}
  where $\lambda^{\ast}(\bm{M})$ denotes the largest singular value of matrix
  $\bm{M}(\bm{\alpha}, \bm{\theta})$. If the largest singular value of $\bm{M}$ is less than $1$, the sequence
  $\left( \abs{\bm{z}^{(n+1)} - \bm{z}^{(n)}}\right)$ converges to zero in $\ell^2$-norm as $n \rightarrow \infty$. Therefore, for $\lambda^{\ast}(\bm{M})<1$, $\ell^2$-norm $\left(  \bm{z}^{(n)}  \right)$ is a Cauchy sequence and must converge. 

  By concatenating local update function \eqref{eq:ratio-update}, we have a global update function $\bm{F} = \left(  F_{ts}  \right)_{ (t \rightarrow s) \in \vec{\Ee}}$, which defines a mapping from $\RR^{\abs{\vec{\Ee}}}$ to $\RR^{\abs{\vec{\Ee}}}$. $\bm{F}$ is a continuous function of $\bm{z}$, we have
  \begin{equation}
    \bm{F}(\lim_{n\rightarrow \infty}\bm{z}^{(n)}) = \lim_{n\rightarrow \infty}\bm{F}(\bm{z}^{(n)}).
  \end{equation}
  Assume that $\left(  \bm{z}^{(n)} \right)$ converges to
  $\bm{z}^{\ast}$. Then
  \begin{align}
    \bm{F}(\bm{z}^{\ast}) - \bm{z}^{\ast}
    &= \lim_{n\rightarrow \infty} \bm{F}(\bm{z}^{(n)}) -\lim_{n\rightarrow
      \infty} \bm{z}^{(n)} \nonumber \\
    &= \lim_{n\rightarrow \infty} (\bm{z}^{(n+1)} - \bm{z}^{(n)}) \nonumber \\
    &= 0.
  \end{align}
  Thus $\bm{z}^{\ast}$ must be a fixed point.

  In what follows we show that the fixed point is unique when $\lambda^{\ast}(\bm{M})<1$. Assume that there are two fixed points $\bm{z}_0^{\ast}$ and $\bm{z}_1^{\ast}$ for sequence $\left\{ \bm{z}^{(n)} \right\}$. Then we have
  \begin{align}\label{eq:two-fix-point}
    \bm{F}(\bm{z}_0^{\ast}) &= \bm{z}_0^{\ast}, \nonumber\\
    \bm{F}(\bm{z}_1^{\ast}) &= \bm{z}_1^{\ast}. 
  \end{align}
  Applying \eqref{eq:ieq-largest-singular} gives
  \begin{equation}\label{eq:two-fix-ineq}
    \normd{\bm{F}(\bm{z}_0^{\ast}) - \bm{F}(\bm{z}_1^{\ast})} \leq
    \lambda^{\ast}(\bm{M})\normd{\bm{z}_0^{\ast} - \bm{z}_1^{\ast}}.
  \end{equation}
  Substituting \eqref{eq:two-fix-point} into \eqref{eq:two-fix-ineq}
  gives
  \begin{equation}
    \normd{\bm{z}_0^{\ast} - \bm{z}_1^{\ast}} \leq
    \lambda^{\ast}(\bm{M})\normd{\bm{z}_0^{\ast} - \bm{z}_1^{\ast}},
  \end{equation}
  which gives us $\bm{z}_0^{\ast} = \bm{z}_1^{\ast}$ and completes the uniqueness of the fixed point.
\end{proof}

\begin{rem}
  From Theorem~\ref{thm:normd} we can see that, the sufficient condition for convergence of $\alpha$-BP is $\lambda^{\ast}(\bm{M}(\bm{\alpha}, \bm{\theta})) < 1$. It is interesting to notice that $\lambda^{\ast}(\bm{M}(\bm{\alpha}, \bm{\theta}))$ is a function of $\bm{\alpha}$ from $\alpha$-divergence and $\bm{\theta}$ from joint distribution $p(\bm{x})$. This means that whether $\alpha$-BP can converge depends on the graph $\Gg= (\Vv, \Ee)$ representing the problem $p(\bm{x})$ and also the choice of $\bm{\alpha}$. Therefore, proper choice of $\bm{\alpha}$ can guarantee the convergence of $\alpha$-BP if the sufficient condition can possibly be achieved for given $\bm{\theta}$.
\end{rem}

\section{Alternative Convergence Conditions for $\alpha$-BP}\label{sec:cvng-coro}
Given the fact that $\alpha$-BP would converge if the condition in Theorem~\ref{thm:normd} is fulfilled, the largest singular value computation for large-sized graph could be nontrivial. In this section, we give alternative sufficient conditions for the convergence of $\alpha$-BP.

\begin{cor}
  $\alpha$-BP would converge to a fixed point if the condition
  \begin{align}
    \umax{u \rightarrow v} & \abs{1-\alpha_{uv}} + \abs{1-\alpha_{vu}}\tanh{(\abs{\alpha_{vu} \theta_{vu}})} 
                           + \sum_{w \in \Nn(v) \backslash u} \tanh{(\abs{\alpha_{{vw}} \theta_{vw}})} < 1,
  \end{align}
  is fulfilled or the condition
  \begin{align}
    \umax{t \rightarrow s} &\abs{1-\alpha_{ts}} (1 + \tanh(\abs{\alpha_{ts} \theta_{ts}})) 
                           + (\abs{\Nn(t)}-1)\tanh(\abs{\alpha_{ts} \theta_{ts}}) < 1.
  \end{align}
  is achieved, where $\abs{\Nn(t)}$ denotes the carnality of the set $\Nn(t)$. The associated fixed point is unique.
\end{cor}

\begin{proof}
  Setting $p=1$ to \eqref{eq:ine-normp-iteration}, we have
  \begin{equation}
    \normu{\bm{z}^{(n+1)} - \bm{z}^{n}} \leq \normu{\bm{M}(\alpha, \bm{\theta})\abs{\bm{z}^{(n)} - \bm{z}^{(n-1)}}}.
  \end{equation}
  Furthermore, from \eqref{eq:ine-normp-iteration}, we also have
  \begin{equation}
    \normi{\bm{z}^{(n+1)} - \bm{z}^{n}} \leq \normi{\bm{M}(\alpha, \bm{\theta})\abs{\bm{z}^{(n)} - \bm{z}^{(n-1)}}},
  \end{equation}
  where $\norm{\cdot}_{\infty}$ denotes the $\ell^{\infty}$-norm.
  Then we have
  \begin{align}\label{eq:row-col-norm-ineq}
    \normu{\bm{z}^{(n+1)} - \bm{z}^{n}} &\leq \normu{\bm{M}}\normu{\bm{z}^{(n)} - \bm{z}^{(n-1)}}, \nonumber \\
    \normi{\bm{z}^{(n+1)} - \bm{z}^{n}} &\leq \normi{\bm{M}}\normi{\bm{z}^{(n)} - \bm{z}^{(n-1)}},
  \end{align}
  where we omit the parameters of $\bm{M}$ here for simplicity. We can expand the first multiplicand on the right hand side of \eqref{eq:row-col-norm-ineq} as follows
  \begin{align}
    \normu{\bm{M}} =& \umax{u \rightarrow v}{\sum_{t\rightarrow s} M_{(t\rightarrow s), (u\rightarrow v)}} \nonumber \\
    =& \umax{u \rightarrow v}\abs{1-\alpha_{uv}} + \abs{1-\alpha_{vu}}\tanh{\abs{\alpha_{vu} \theta_{vu}}} 
                    + \sum_{w \Nn(v) \backslash u} \tanh{\abs{\alpha_{{vw}} \theta_{vw}}}, \nonumber \\
    \normi{\bm{M}} =& \umax{t \rightarrow s}{\sum_{u\rightarrow v} M_{(t\rightarrow s), (u\rightarrow v)}} \nonumber \\
    =& \umax{t \rightarrow s} \abs{1-\alpha_{ts}} (1 + \tanh\abs{\alpha_{ts} \theta_{ts}}) 
                    + (\abs{\Nn(t)}-1)\tanh\abs{\alpha_{ts} \theta_{ts}}.
  \end{align}
  When condition $\normu{\bm{M}} < 1$ is met, sequence $\left( \abs{\bm{z}^{(n+1)} - \bm{z}^{n}} \right)$ approaches to zero as $n\rightarrow \infty$. Similarly, condition $\normi{\bm{M}}<1$ can also guarantee the convergence to zero of sequence $\left( \abs{\bm{z}^{(n+1)} - \bm{z}^{n}} \right)$. The analysis for uniqueness of converged fixed point is similar to that in proof of Theorem~\ref{thm:normd}.
\end{proof}

\begin{figure*}[!t]
  \subfigure[]{\label{fig:largest_singular}
    \includegraphics[width=0.33\columnwidth]{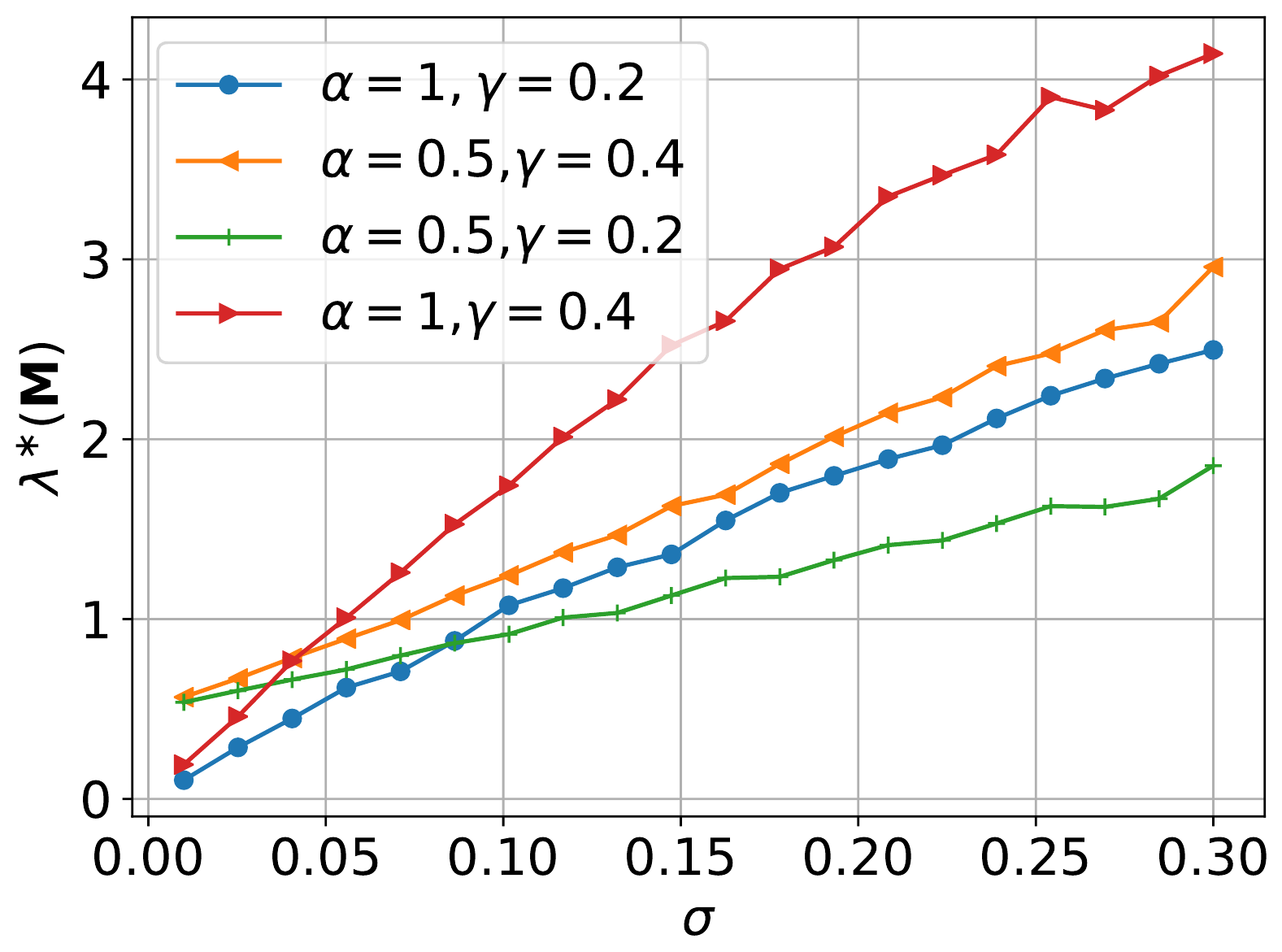}}
  \subfigure[]{\label{fig:log-error-iter-diverse}
    \includegraphics[width=0.33\columnwidth]
    {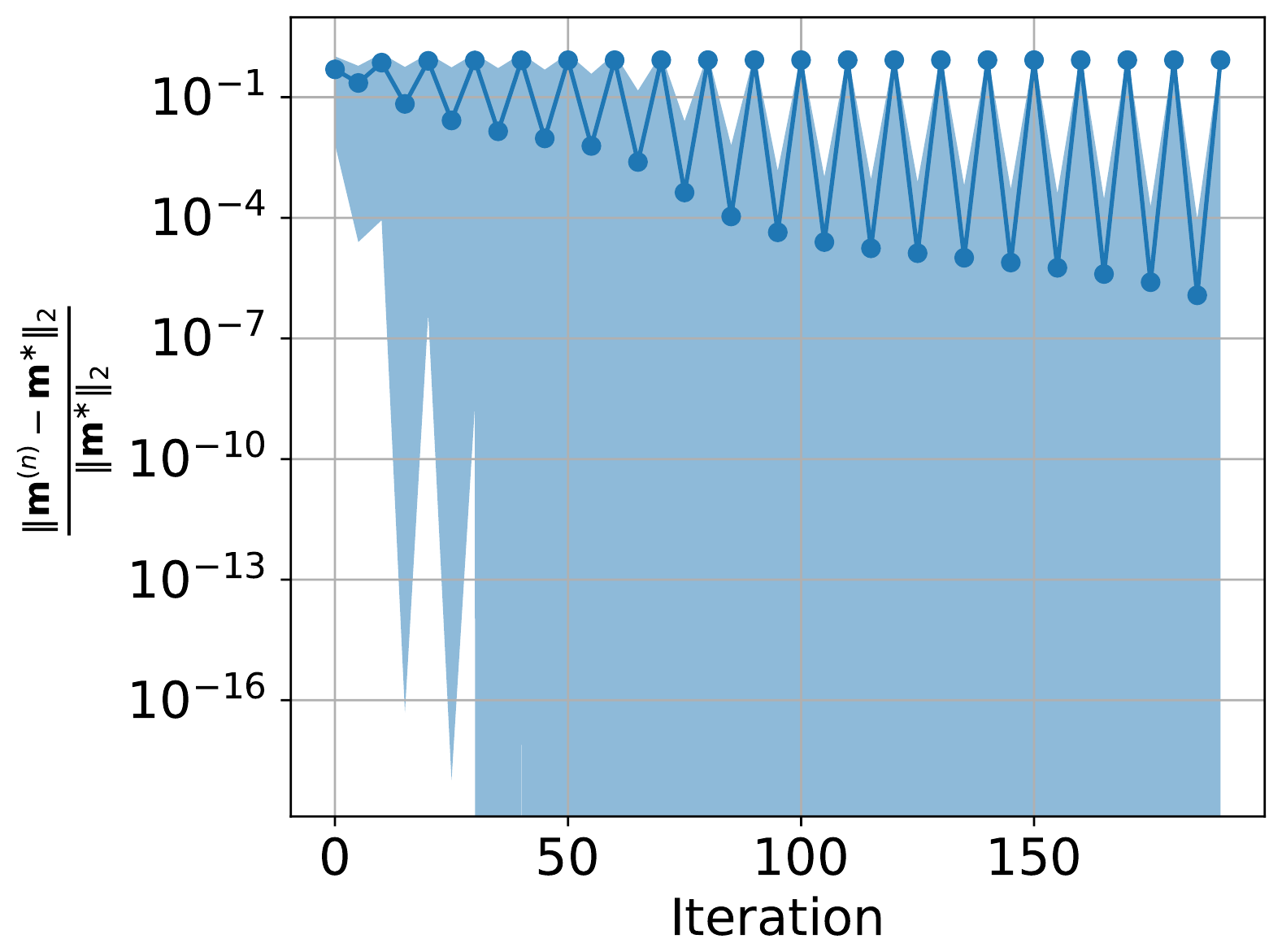}}
  \subfigure[]{\label{fig:log-error-iter-converge}
    \includegraphics[width=0.33\columnwidth]
    {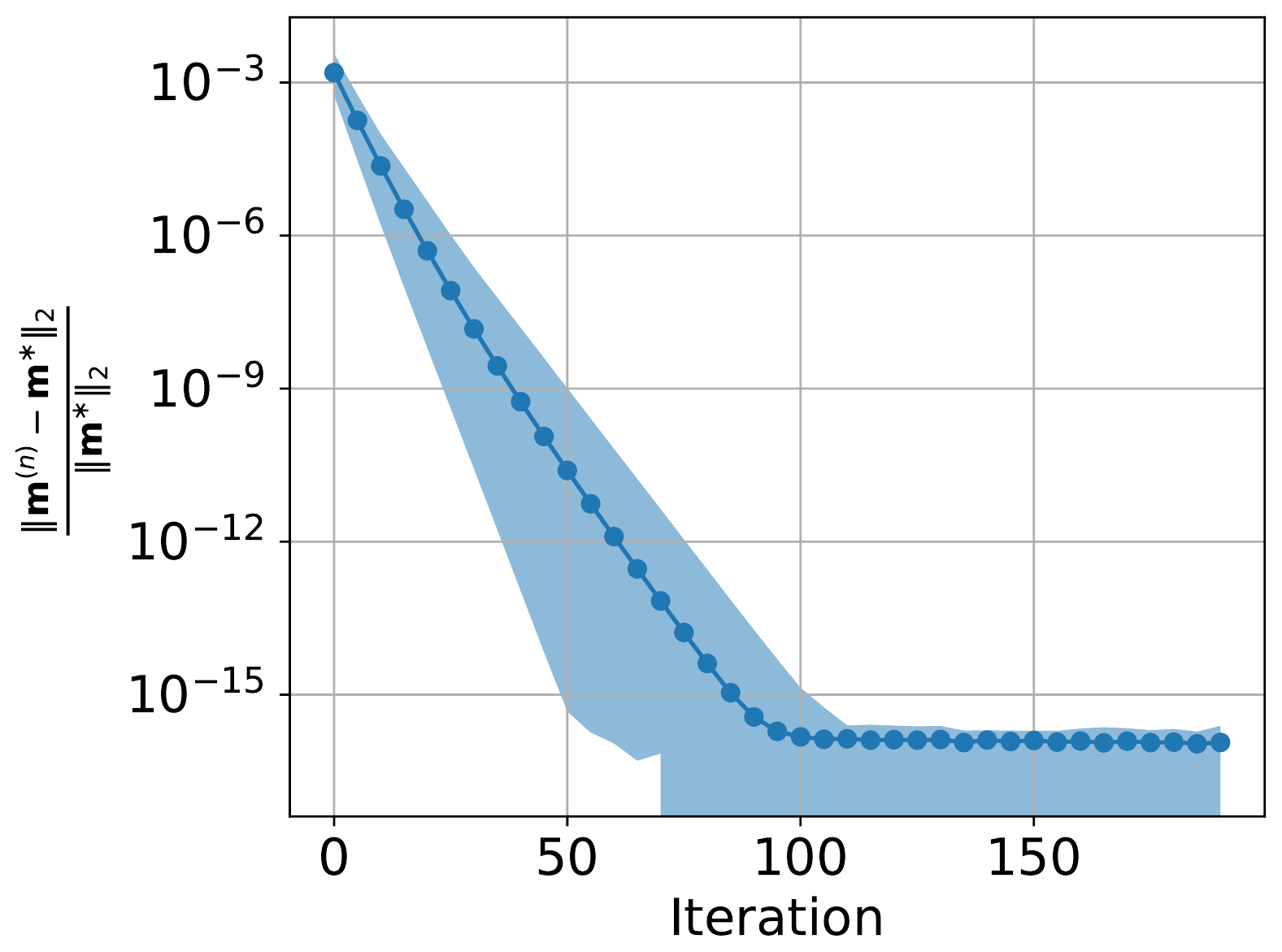}}
  
  \caption{Numerical results on convergence, number of nodes $N=16$: ({a}) the largest singular value of $\bm{M}$ defined in \eqref{eq:matrix_m} versus variance of potential parameter $\bm{\theta}$. ({b}) Parameter setting: $\gamma =
    0.4$, $\alpha = 1$ (equivalent to standard BP), $\sigma = 0.5$.
    and ({c}) Parameter setting: $\gamma =
    0.2$, $\alpha = 0.5$, $\sigma = 0.1$.
    For (b) and (c), normalized error ${\normd{\bm{m}^{(n)}-\bm{m}^{\ast}}}/{\normd{\bm{m}^{\ast}}}$
    versus the number of iterations, blue region denotes the range from minimum to maximum of the normalized error of $100$ graph realizations, whereas the curve stands for mean error of the $100$ realized graphs. }
  \label{fig:mimo_detection}
  
\end{figure*}

\section{Experiments}

In this section, we firstly give the simulations for convergence condition of $\alpha$-BP explained in Theorem~\ref{thm:normd}. Then the application of $\alpha$-BP to a MIMO detection problem is demonstrated. Code is available at https://github.com/FirstHandScientist/AlphaBP.

\subsection{Simulations on Synthetic Problem}

In this section simulations on random graphs are carried out
to gain some insights on the $\alpha$-BP. The random graphs used here are generated
by Erdos-R\'enyi (ER) model \cite{erdos1960}. In generating a graph by ER model, an edge between any two nodes is generated with probability $\gamma$, $\gamma \in (0,1)$.

Note that the MRF joint probability in \eqref{eq:mrf} can be reformulated into
\begin{equation}\label{eq:mrf_matrix_form}
  p(\bm{x}) \propto \exp\{-\bm{x}^{T}\bm{J}\bm{x} - \bm{b}^{T}\bm{x}\}, \bm{x} \in \Aa^N,
\end{equation}
with $\phi_{ts}(x_t, x_s) = e^{- 2 J_{t,s} x_t x_s}$ and $\phi_s(x_s) = e^{ - J_{s,s} - b_s  x_s}$. $J$ here is the weighted adjacency matrix. In our experiments, we generate a random graph $\Gg = (\Vv, \Ee)$ with $\gamma$ by ER model and then associate potential factors to the graph. Specifically, factor $\phi_s(x_s)$ is associated to node $x_s$, $s \in \Vv$, and $\phi_{ts}$ to edge $(t, s) \in \Ee$. $J_{ts}$ is zero if these is no edge $(t, s)$.

For this set of experiments, we set $\Aa = \left\{ -1, 1 \right\}$ and $N=16$. To specify \eqref{eq:mrf_matrix_form}, the non-zero entries of $\bm{J}$ is sampled from a Gaussian distribution $\mathsf{N}(0, \sigma^2)$, i.e. $J_{ts} \sim \mathsf{N}(0, \sigma^2)$ if $J_{ts}\neq 0$. For entries of $\bm{b}$, we use $b_t \in \mathsf{N}(0, (\sigma/4)^2)$.
For each edge $(t,s) \in \Ee$, we set $\alpha_{ts}=\alpha$, i.e. the edges share a global value $\alpha$.

In Figure~\ref{fig:largest_singular}, we show how the largest singular value of $\bm{M}(\bm{\alpha}, \bm{\theta})$ as defined in \eqref{eq:matrix_m} changes when the standard deviation $\sigma$ of potential factors increases. The behavior is illustrated with different values of $\alpha$ and the edge probability $\gamma$. For each curve, a point on the curve is the mean of $100$ realizations of random graphs as described above. The curves of Figure~\ref{fig:largest_singular} show in general that a larger standard deviation of potential factors of graph edges makes it more difficult to fulfill the convergence condition in Theorem~\ref{thm:normd}. This is also the case when a graph is denser as we raise the edge probability $\gamma$ in generating random graphs, by comparing the green and orange curves. The comparison between green and blue curves indicates that choice of $\alpha$ value in $\alpha$-BP also makes a difference, and its effect depends on the graph itself. How to tune $\alpha$ value to fulfill condition of Theorem~\ref{thm:normd} depends not only how dense ($\gamma$) the graph is, but also how potential factors spread out from each other.

To illustrate our developed convergence condition for $\alpha$-BP, we also observe how messages in a graph changes along belief propagation iterations. To be specific, we run our $\alpha$-BP with $200$ iterations on a graph, after which the messages in the graph are denoted by $\bm{m}^{\ast}$. $\bm{m}^{\ast}$ can be the converged messages if $\alpha$-BP has converged within the $200$ iterations. Then we measure the quantity $\normd{\bm{m}^{(n)} - \bm{m}^{\ast}}/\normd{\bm{m}^{\ast}}$ during the iterations. In Figure~\ref{fig:log-error-iter-diverse}, we generate $100$ random graphs by ER model with parameter setting as $\gamma =0.4$, $\alpha = 1$\footnote{$\alpha=1$ in $\alpha$-BP corresponding to standard BP.}, $\sigma = 0.5$. By referring to the curves in Figure~\ref{fig:largest_singular}, it can be inferred that this setting does not fulfill the condition in Theorem~\ref{thm:normd}. The log error changes versus iteration number $n$ for the $100$ graphs are shown in Figure~\ref{fig:log-error-iter-diverse}, in which the blue region indicates the range and the solid curve indicates the mean of the normalized errors. It is clear that Figure~\ref{fig:log-error-iter-diverse} does not show any sign of convergence within $200$ iterations.

\begin{figure*}[!t]
  \subfigure[Standalone $\alpha$-BP without any prior~]{\label{fig:mimo_a}
    \includegraphics[width=0.33\columnwidth]{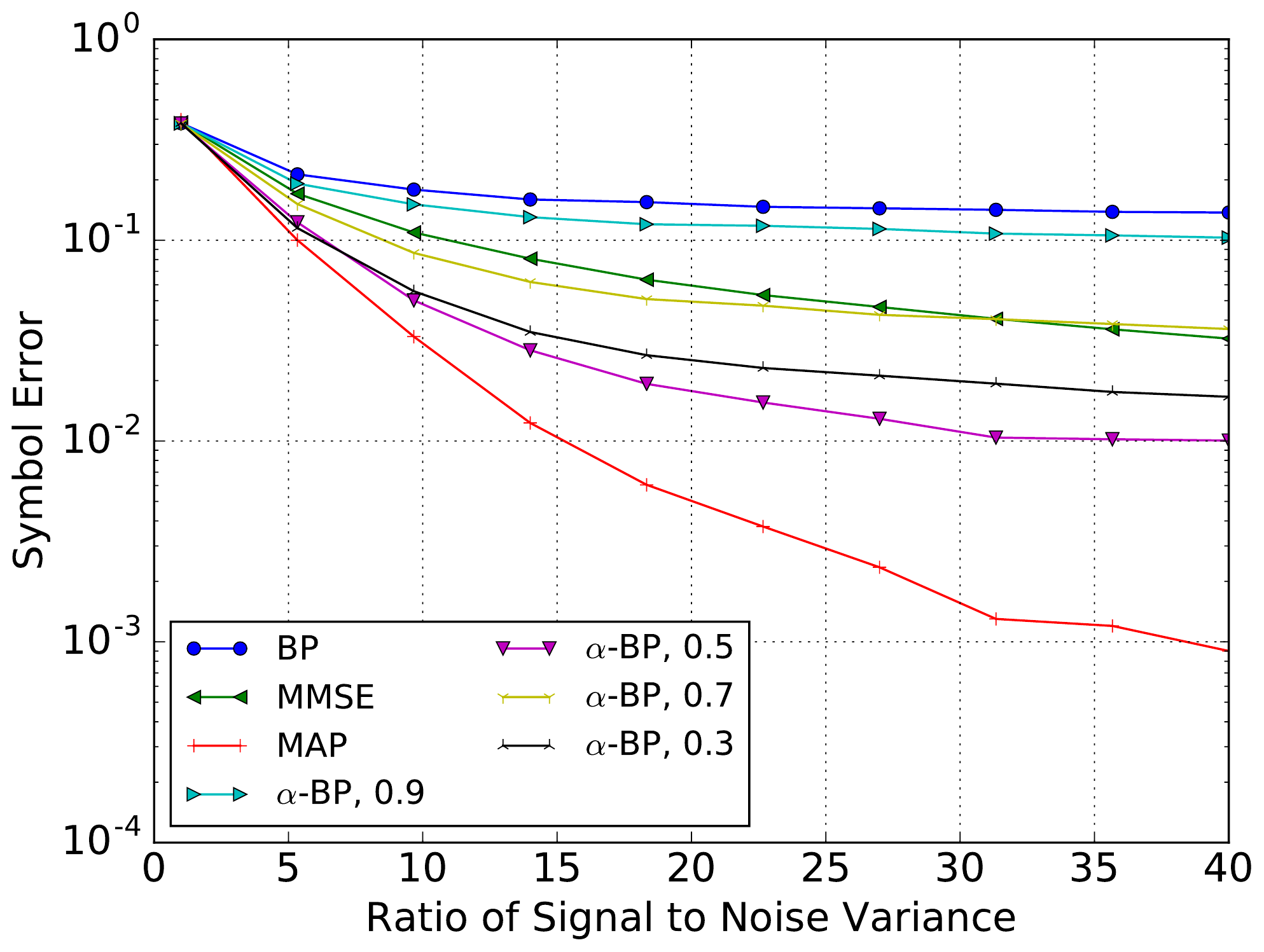}}
  \subfigure[Incorporate prior into $\alpha$-BP~]{\label{fig:mimo_b}
    \includegraphics[width=0.33\columnwidth]{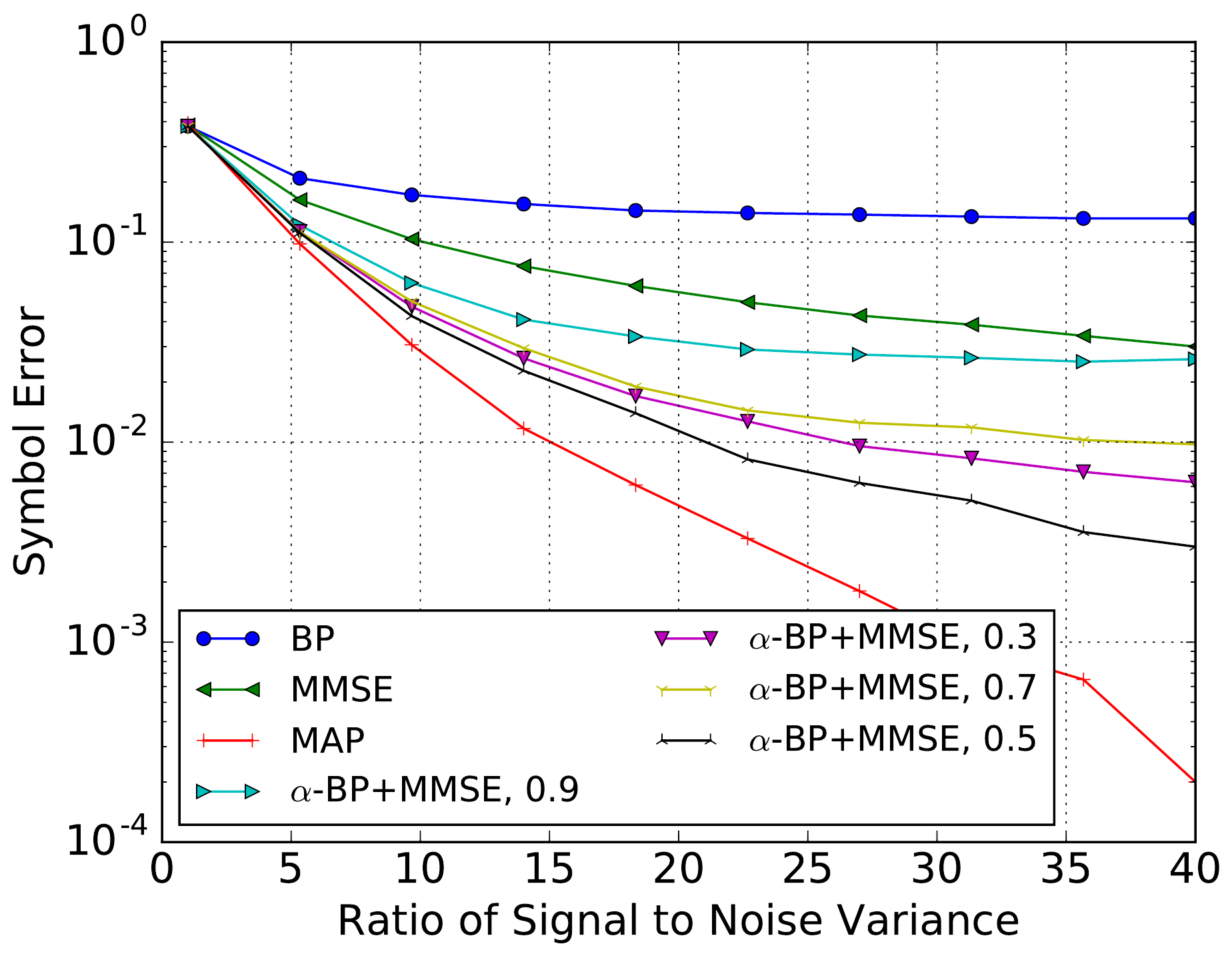}}
  \subfigure[Comparison with other message-passing methods]{\label{fig:mimo_c}
    \includegraphics[width=0.33\columnwidth]{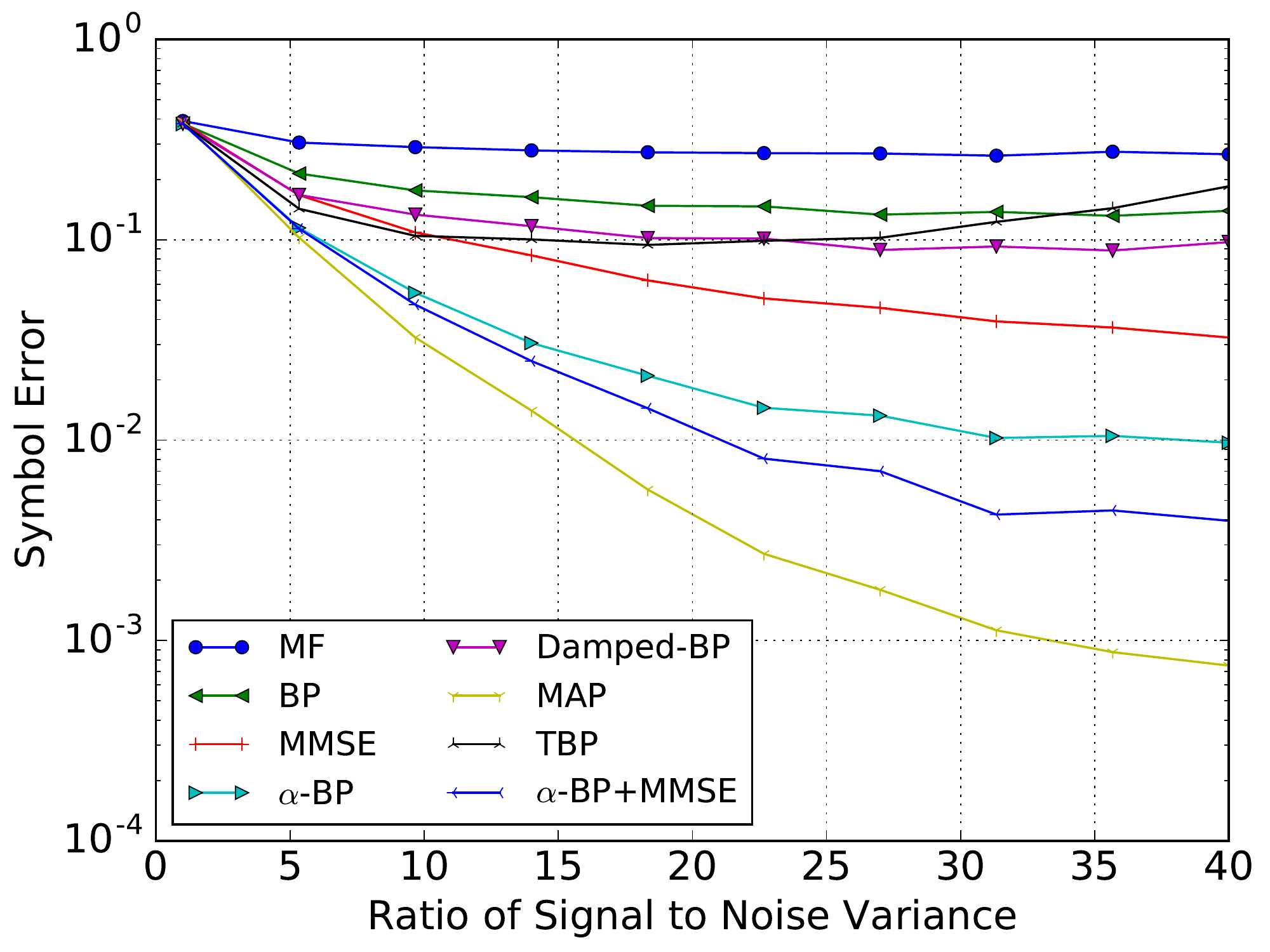}}
  
  \caption{Numerical results of $\alpha$-BP: symbol error of MIMO detection.}
  \label{fig:mimo_detection}
  
\end{figure*}

We then carry out a set of experiments in Figure~\ref{fig:log-error-iter-converge} similar to our experiments in Figure~\ref{fig:log-error-iter-diverse}. The only difference lies in the graph generating process. Here we set the parameters to be $\gamma =0.2$, $\alpha = 0.5$, $\sigma = 0.1$. According to our curves in Figure~\ref{fig:largest_singular}, a graph generated with this parameter setting should fulfill the condition in Theorem~\ref{thm:normd}. Due to randomness of both graph generating by ER and potential factors, we regenerate a graph if the initial generated graph does not satisfy $\lambda^{\ast}(\bm{M})<1$. Therefore the $100$ graphs used in experiments for Figure~\ref{fig:log-error-iter-converge} all fulfill the Theorem~\ref{thm:normd}. The result in Figure~\ref{fig:log-error-iter-converge} is consistent with our analysis on the convergence of $\alpha$-BP.

\subsection{Application to MIMO Detection}

In this section, we show the application of $\alpha$-BP to a MIMO detection problem. For a MIMO system, the observation $\bm{y}$ is a linear function of channel $\bm{H}\in \RR^{N\times N}$ when unknown signal $\bm{x}$ need to be estimated,
\begin{equation}\label{eq:linear-model}
  \bm{y} = \bm{H} \bm{x} + \bm{e}, \bm{x} \in \Aa^N,
\end{equation}
where $\bm{e}$ is  noise modeled as Gaussian noise $ \bm{e} \sim \mathsf{N}\left( \bm{0}, \sigma^2_{w} \bm{I} \right)$. Here $\bm{I}$ is unitary matrix. In this case, the posterior of $\bm{x}$ can be written as:
\begin{align}\label{eq:true-posterior}
  p(\bm{x}|\bm{y}) &\propto e^{ - \frac{1}{2\sigma_w^2} \normd{\bm{Hx} - \bm{y}}^2 } \nonumber \\
                   & = e^{ - \frac{1}{2\sigma_w^2}\left[ \bm{x}^{T}\bm{H}^{T}\bm{H}\bm{x} - 2 \bm{y}^T\bm{H}\bm{x}  + \bm{y^T}\bm{y}  \right] }.
\end{align}
Denote $\bm{S} = \bm{H}^T\bm{H}$, $\bm{h}_i$ as the $i$-th column of $\bm{H}$, and
\begin{equation}
  \hspace{-0.3cm}\phi_{i}(x_i) = e^{- \frac{S_{i,i} x_i^2}{2 \sigma_w^2} + \frac{\langle {\bm{h}_i, \bm{y}}\rangle x_i}{\sigma_w^2} },
  \phi_{ij}(x_i, x_j) = e^{ -\frac{x_i S_{i,j} x_j}{\sigma_w^2} }.
\end{equation}
Then \eqref{eq:true-posterior} is a realization of \eqref{eq:mrf}. We set $\Aa = \left\{ -1, 1 \right\}$, $N = 8$, and $\bm{H}\in \RR^{8 \times 8}$ sampled from Gaussian.

We test the application of $\alpha$-BP to the MIMO signal detection numerically. We run the $\alpha$-BP, without the prior trick (Subsection~\ref{subsec:remark}) in Figure~\ref{fig:mimo_a} and with the prior in Figure~\ref{fig:mimo_b} (legend ``$\alpha$-BP$+$MMSE'') as estimation of minimum mean square error (MMSE). The reference results of MMSE and maximum a posterior (MAP, exhausted search) are also reported under the same conditions. MMSE estimator depends on Gaussian posterior $\mathsf{N}(\hat{\bm{\mu}}, \hat{\bm{\Sigma}})$, where $\hat{\bm{\mu}} = (\bm{H}^{T}\bm{H} + \sigma_w^2 \bm{I})^{-1}\bm{H}^{T}\bm{y}$ and $\hat{\bm{\Sigma}} = (\bm{H}^{T}\bm{H} + \sigma_w^2 \bm{I})^{-1}\sigma_w$. Detection of MMSE carried out by $\argmin_{x_i\in \Aa}\abs{x_i-\hat{\mu}_i}$.

Figure~\ref{fig:mimo_a} shows that BP even underperforms MMSE but $\alpha$-BP can outperform MMSE by assigning smaller value of $\alpha$.
Note that MMSE requires the matrix inverse computation whose complexity is proportional to $N^3$, while the complexity of $\alpha$-BP increases linearly with $N$. Therefore $\alpha$-BP is superior to MMSE both performance-wise and complexity-wise.  
However, there is still a big gap between $\alpha$-BP (even for $\alpha=0.5$) and MAP. This gap can be decreased further by using the prior trick discussed in Subsection~\ref{subsec:remark}. Figure~\ref{fig:mimo_b} exemplifies this effects by using prior belief from MMSE, $\hat{p}_i(x_i)\propto \exp\{-(x_i-\hat{\mu}_i)^2/(2\hat{\bm{\Sigma}}_{i,i})\}$, by modifying the graph as shown in Figure~\ref{fig:factor-graph-with-prior}, which comes with legend "$\alpha$-BP$+$MMSE". It is shown that larger performance gain is observed when $\alpha$-BP runs with prior belief.

Additional, we also carry out the experiments where the proposed $\alpha$-BP is compared with mean field (legend 'MF'), BP with damping technique \cite{Pretti2005damping} (with legend 'Damped-BP'), and Tree-reweighted belief propagation \cite{wainwright2008graphical} (with legend 'TBP') in Figure~\ref{fig:mimo_c}. As expected, mean field method performs no better than BP. Damping technique improves BP's performance with a noticeable difference but still falls behind MMSE. The performance of tree-reweighted BP reaches that of MMSE in low ratio range of signal-to-noise variance but degenerates a lot in the high ratio range. The old message and potential factors are reweighted by the edge appearance probability in TBP to compute new messages. In TBP, the edge appearance probability is the probability that the edge exists in a randomly chosen spanning tree from all possible spanning trees of graph $\Gg$, which is usually expensive to compute.

\section{Conclusion}
In this paper, we have developed and analyzed a new alternative to standard BP message passing algorithm. The developed $\alpha$-BP algorithm has the clear intuition of minimization of a localized $\alpha$-divergence. Convergence conditions of $\alpha$-BP are offered in binary state space of each random variable. $\alpha$-BP is a valid and practical algorithm accoinding to our experiments. With prior trick, the performance of $\alpha$-BP can be further improved. Future works would be to investigate more general conditions for convergence of $\alpha$-BP.

\bibliographystyle{plain}
\bibliography{myref}
\appendix
\section{Related Methods}
\label{apdx:sec:related-msg-passings}
\subsection{Mean Field}
Although mean field method is developed also based on the assumption of fully factorized approximattion, it is derived from a different way against $\alpha$-BP. As explained in \cite{yedida2003understanding}, the mean field method can be obtained from minimizing a Gibbs free energy
\begin{equation}\label{apdx:eq:mf-energy}
  F_{MF}(q_{MF}(\bm{x})) = \sum_{\bm{x}} q_{MF}(\bm{x}) \log{\frac{q_{MF}(\bm{x})}{\tilde{p}(\bm{x})}},
\end{equation}
with $q_{MF}(\bm{x})$ defined as
\begin{equation}
  q_{MF}(\bm{x}) = \prod_{s\in \Vv} q_s(x_s).
\end{equation}
Then \eqref{apdx:eq:mf-energy} can be rewritten as
\begin{align}
  F_{MF}(q_{MF}(\bm{x})) = &\sum_{s\in\Vv}\sum_{x_s} q_s(x_s) \ln{q_s(x_s)} \nonumber \\
  &- \sum_{s\in\Vv}\sum_{x_s} q_s(x_s) \ln{\phi_s(x_s)} \nonumber \\
  &- \sum_{(s,t)\in\Ee}\sum_{x_s} q_s(x_s) q_t(x_t) \ln{\phi_{st}(x_s,x_t)}.
\end{align}
The mean field method can be obtained from solving $\min_{q_s(x_s)}F_{MF}(q_{MF}(\bm{x}))$ for all $\left\{ q_s(x_s) \right\}$ alternatively.

\subsection{BP}
As is known, standard belief propagation has the intuition as the minimization of Bethe free energy \cite{yedidia2005constructing}.
In Bethe approximation, pairwise-node joint approximation $\left\{ q_{st}(x_s, x_t) \right\}$ is considered apart from the single-node approximation $\{q_s(x_s)\}$. The Bethe approximation takes the form of
\begin{equation}
  q_B(\bm{x}) = \frac{\prod_{(s,t)\in \Ee}q_{st}(x_s, x_t)}{\prod_{s\in\Vv}q_s(x_s)^{N_s-1}},
\end{equation}
where $N_s$ is the number of neighbors of node $s$ in the factor graph.
In this case the variational free energy named Bethe free energy is
\begin{align}
  &F_B(q_B(\bm{x})) = \sum_{(s,t)\in\Ee}\sum_{x_s, x_t} q_{st}(x_s, x_t) \ln\frac{q_{st}(x_s,x_t)}{\phi_{st}(x_s, x_t)} \nonumber \\
  &+ \sum_{s\in\Vv}\sum_{x_s}q_s(x_s) \ln{\frac{q_s(x_s)}{\phi_s(x_s)}} - \sum_{s\in\Vv}N_s\sum_{x_s} q_s(x_s) \ln{{q_s(x_s)}}.
\end{align}

The minimization of $F_B$, with marginalization constraints $\sum_{x_s}q_{st}(x_s, x_t) = q_t(x_t)$, $\forall~ (s,t)\in \Ee $, recovers the message update rule of standard BP as in \eqref{eq:bp-update-rule} in our paper, i.e.
\begin{equation}\label{apdx:eq:std-bp}
  {m}^{\text{new}}_{ts}(x_s) \propto 
  \sum_{x_t} \phi_{st}(x_s, x_t) {\phi}_t(x_t) \prod_{w\in \Nn(t)\backslash s}m_{wt}(x_t).
\end{equation}

As can be seen by comparing \eqref{eq:message-rule} in our paper with \eqref{apdx:eq:std-bp} here, the message passing rule of standard BP can be recovered from that of $\alpha$-BP by setting all edges' $\alpha$ value to be $1$.

The damping technique is to update the message by a soft combination instead of directly applying the update message to be the new message. The damping technique is usually applied to standard BP in order to help the convergence. To clearly state the difference between message update rule of BP with damping (damped BP) and that of $\alpha$-BP, we give the message update of damped BP
\begin{equation}
  {m}^{\text{new}}_{ts}(x_s) \propto {m}_{ts}(x_s)^{\gamma}\bigg[
  \sum_{x_t} \phi_{st}(x_s, x_t) {\phi}_t(x_t)\!\!\!\!\!\! \prod_{w\in \Nn(t)\backslash s}\!\!\!\!\!\!m_{wt}(x_t) \bigg]^{1-\gamma}
\end{equation}
that does soft update in product way, where $0< \gamma < 1$.

\subsection{Tree-reweighted BP}
The tree-reweighted BP shares some similarity with $\alpha$-BP in formula of the message-passing rule, namely the pairwise log-potential functions are scaled by a weight and reweighted old messages appear in computation of new messages. But different from $\alpha$-BP, tree-reweighted BP is derived by obtaining an upper bound of log-partition function of $p(\bm{x})$ first via a Jensen's inequality and minimize the upper bound \cite{wainwright2008graphical, wainwright2002tree}. The upper bound is
\begin{align}
  & F_T = \sum_{s\in\Vv}\sum_{x_s}q_s(x_s) \ln{\frac{q_s(x_s)}{\phi_s(x_s)}} \nonumber \\
  & + \sum_{(s,t)\in\Ee}\mu_{st}\sum_{x_s, x_t} q_{st}(x_s, x_t) \ln\frac{q_{st}(x_s,x_t)}{q_s(x_s)q_t(x_t)} \nonumber \\
  & - \sum_{(s,t)\in\Ee}\sum_{x_s, x_t} q_{st}(x_s, x_t)\ln{\phi_{st}(x_t, x_t)},
\end{align}
where $0\leq \mu_{st} \leq 1$ is defined as the appearance probability of edge $(s, t) \in \Ee$, which denotes the appearance rate of edge $(s,t)$ among all spanning trees of graph $\Gg$. Denotes set of all spanning trees of $\Gg$ by $\Tt(\Gg)$. $\mu_{st}$ is the probability that edge $(s,t)$ exists in a randomly selected spanning tree from $\Tt(\Gg)$. The appearance rate can be expensive to compute as it is defined on all spanning trees of a graph.

The upper bound $F_T$ can be reduced into $F_B$ when $\mu_{st}=1$, $\forall ~(s,t)\in \Ee$. The message-passing updates of the tree-reweighted algorithm corresponds to the minimization of $F_T$ with marginalization constraints, which can be written as
\begin{align}
    &{m}^{\text{new}}_{ts}(x_s) \propto  \nonumber \\
  &\sum_{x_t} \phi_{st}(x_s, x_t)^{1/\mu_{st}} {\phi}_t(x_t) \frac{\prod_{w\in \Nn(t)\backslash s}m_{wt}(x_t)^{\mu_{wt}} }{m_{st}(x_t)^{1-\mu_{st}}}.
\end{align}
In the message update rule, both pairwise potential factor and old messages are reweighted, which are different from the way of how pairwise potential factor and old message are reweighted in message update in \eqref{eq:message-rule} of $\alpha$-BP in our paper. Nevertheless, $\alpha$-BP is derived in the way that is different from tree-reweighted BP.

\section{More experimental results}
\begin{figure}[!t]
  \subfigure[More comparison with damped BP]{\label{fig:compare-vs-damp}
    \includegraphics[width=0.45\columnwidth]{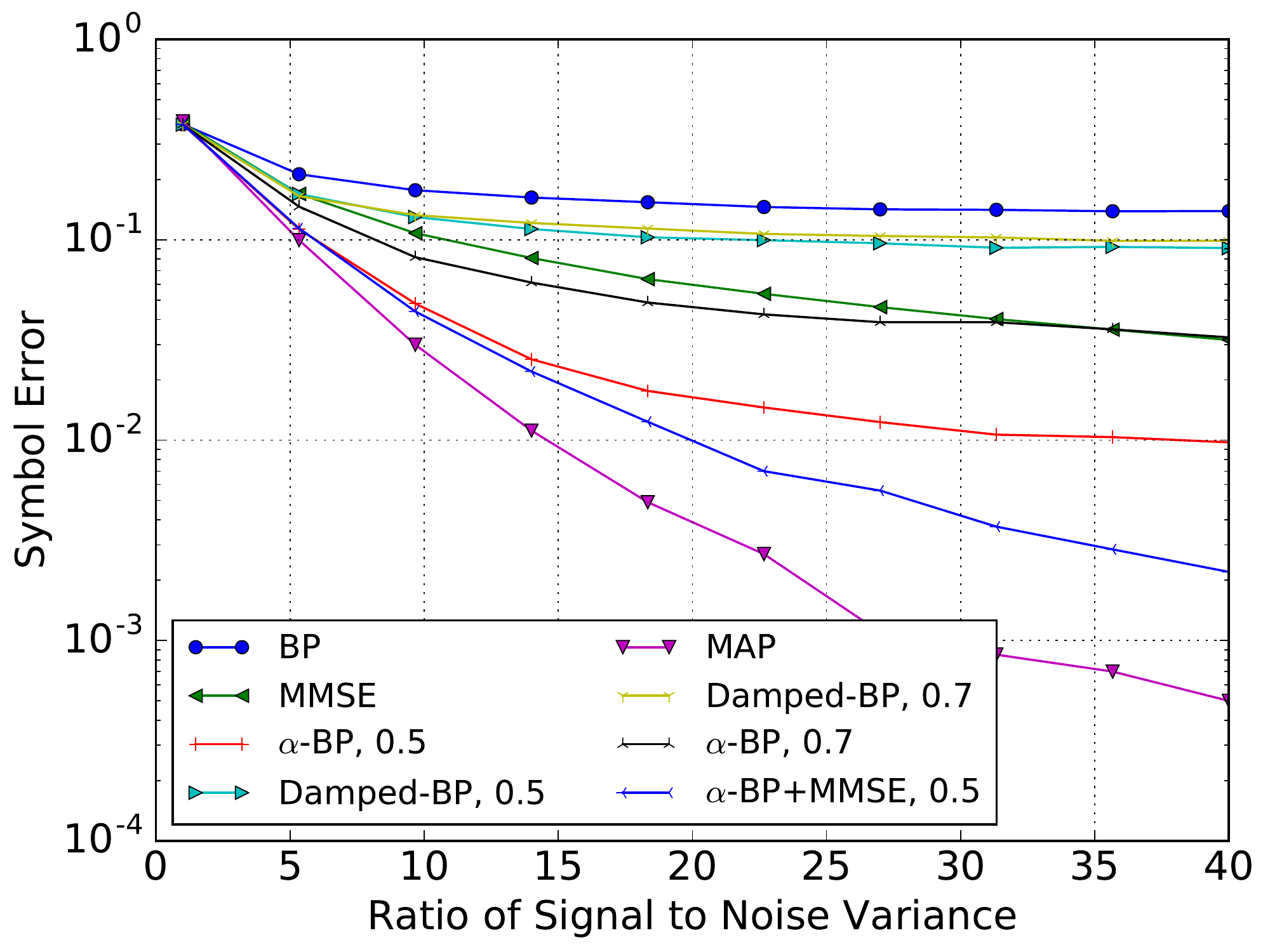}}
  \subfigure[Anneal the alpha value in message iterations]{\label{fig:anneal}
    \includegraphics[width=0.45\columnwidth]{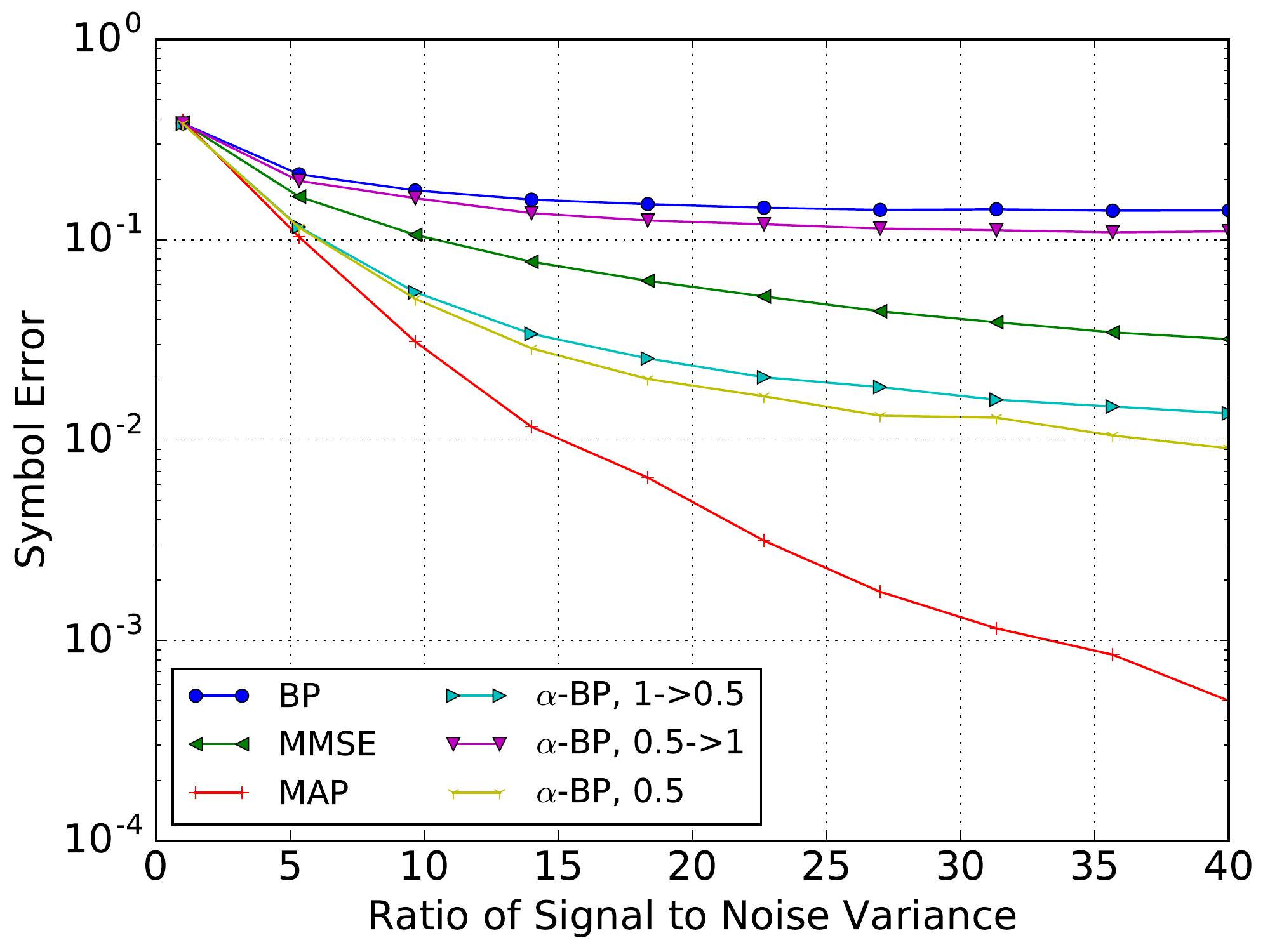}}
  \caption{Some more experimental results}
  \label{fig:extra_mimo_detection}
\end{figure}
More controlled experiments are carried out as shown in Figure~\ref{fig:extra_mimo_detection}. In Figure~\ref{fig:compare-vs-damp}, comparison with damped BP with $\gamma$ setting as $0.5$ and $0.7$ is made. It can be seen that performance of damped BP is better than BP slightly.
 In Figure~\ref{fig:anneal}, we shows the results of annealing $\alpha$ value of $\alpha$-BP as messages are updated. Two alternatives are carried out with $\alpha$ changing from $1$ to $0.5$ (with legend $1\rightarrow 0.5$) and from $0.5\rightarrow 1$ (with legend $0.5 \rightarrow 1$) in a gradual way as messages are iterated. The comparison shows that the values of $\alpha$ make a larger difference at ending iterations than that at beginning iterations.

\end{document}